\documentclass{iopjournal}

\usepackage{amsmath}   
\usepackage{amssymb}   
\usepackage{booktabs}  
\usepackage{xcolor}
\usepackage[linesnumbered,ruled,vlined]{algorithm2e}

\SetCommentSty{mycommfont}
\SetKwInput{KwInput}{Input}                
\SetKwInput{KwOutput}{Output}              
\usepackage{amsmath}  
\usepackage{amsthm}   

\newtheorem{theorem}{Theorem}
\usepackage{pbox}
\usepackage{cite}

\begin{document}

\articletype{Paper} 

\title{Gibbs randomness-compression proposition}

\author{M S\"uzen  \orcid{0000-0002-9460-7297}} \\
\affil{$^1$Member, American Physical Society, College Park, Maryland, United States} \\
\affil{$^2$Resident Scientist, Assia, CY 5561, Cyprus} \\
\email{mehmet.suzen@physics.org}


\begin{abstract}
A proposition that connects randomness and compression is put 
forward via Gibbs entropy over set of measurement vectors associated 
with a lossy compression process. In building this connection, we use 
a performance of a learning task as a probe of compression in iterative
compress-train cycles. This can be thought as iterative coarse-graining 
from statistical mechanics perspective using thermodynamic efficiency as 
a probe. We formulate this connection via comonotonic relationship within 
a very small decrease in compression ratio and the performance. We have 
showcase the validity of this proposition with a canonical vision task in 
deep learning with three different model compression processes 
as {\it a baseline model}. We use the following, simpler to more 
complex model compression approaches: (1) random pruning,(2) 
magnitude pruning, and (3) a more complex compression by using 
dual tomographic compression, which utilizes compressed sensing 
in dual fashion which is introduced as a new method. We use 
remaining weights of deep learning network as a measurement vector 
where we measure the Gibbs entropy. We show case the idea that there 
is an inherent computable connection between compression probed by 
performance and randomness from an entropy measure on the learned
model.
\end{abstract}

\section{Introduction}

Connections between randomness and data compression are noted earlier 
within the algorithmic information framework of Solomonoff-Kolmogorov-Chaitin 
\cite{solomonoff64, kolmogorov65, chaitin66} (SKC). These works aim at 
reaching a definition of randomness from model selection perspective 
\cite{chaitin75}, i.e., algorithmic entropy or a complexity measure. Striking 
connections to algorithmic entropy and Gibbs entropy are also explored from 
fundamental conceptual point of view \cite{bennett82, zurek89, machta99a}. 
Statistical mechanics connections and information are established between 
Gibbs and Boltzmann entropies \cite{jaynes65a,jaynes57a, buonsante16, 
rondoni00, wu26}. However, relationship among compression process, entropy, 
information and randomness is always made quite generic in SKC framework. 
This manifestation was by design quite generic without strict mathematical 
condition, i.e., aimed at finding the shortest algorithm or explanation that 
can generate the random sequence. In this work, we propose a mathematical 
condition among these concepts in a more formal manner. As interplay between 
physics and deep machine learning is now well-established \cite{sherrington, 
hopfield82, ackley85, welling26}.

Moreover, randomness appear as a critical component of contemporary needs
such as the certification of randomness \cite{amer25}, understanding the 
stability of learning in modern models \cite{pecher24}, its relationship 
between quantum complexity \cite{guo24}, improving randomness in quantum 
information \cite{kulikov26}, and its relation to entropy in natural 
sciences \cite{luecher25}. Consequently, building quantitative bridge 
between randomness and a compression process is also a practical interest 
beyond theoretical advancement.

In Section \ref{sec:back}, we introduce the background. Then, we introduce 
inverse compressed sensing in Section \ref{sec:inversecs}. In Section 
\ref{sec:dtc}, we introduced dual tomographic compression, as a simple 
neuron level pruning approach, that is more complex than both random and 
magnitude pruning. We introduce the formal proposition and the numerical 
example as a baseline model in Section \ref{sec:gibbs} and \ref{sec:exp}
respectively. In Section \ref{sec:con}, we conclude our study.

\section{Background} \label{sec:back}

The definition of randomness \cite{compagner91} and quality of generating
random numbers \cite{oneill14} appears to be connected subjects within the
complexity, compressing data \cite{zurek89}, and SKC algorithmic randomness 
context. This framework commonly known as Kolmogorov Complexity (KC), 
algorithmic complexity or algorithmic randomness and its applications 
are vast \cite{li08, zurek91}. KC or SKC framework, makes a generic 
statement about information content of sequence $s$ of binary digits of
length $N$, $K(s)$, defined to be the size of the smallest program in bits 
that generates $s$, the sequence. The connection to compression is implicit 
here and its mathematical definition to compression process appears as only 
a single compression step. It is hinted that there is a positive correlation 
between compression ratio and randomness \cite{li08}. Moreover, the interest 
of randomness in this manner also explored from biological evolution stand 
point \cite{wagner12}, which is more relevant in our case due to proposed 
train-compress framework having mutation-like process. We have summarized
the concepts used here in Table \ref{tab:concept}, with 
a mapping between statistical physics and machine learning with statistical
signal processing. This table helps to clarify the concepts used 
in the paper.

Compressed sensing (CS) framework has shown huge leap in statistical signal
processing and information theory as it shows one can reconstruct an unknown 
signal with much less measurements under sparsity constraints \cite{donoho06a, 
candes06, candes06tao, baraniuk07, eldar12a}. CS involves random sampling and 
it provides an excellent tool to investigate the relationship between compression 
and randomness.

\begin{table}[htbp]
\centering
\begin{tabular}{lll}
\toprule
\textbf{Machine and Statistical Learning} & 
\textbf{Statistical Physics}              & 
\textbf{Notes} \\ 
\midrule
Learn (Train)                          &  
\pbox{20cm}{Thermalization \\ 
            Finding ground state}      & 
\pbox{20cm}{Minimizing Free Energy \\ 
            Optimal model weights} \\
\midrule
\pbox{20cm}{(Model) Compression \\ 
            Complexity Reduction}      & 
Coarse-Graining                        & 
\pbox{20cm}{Reduction In: \\
            Information content \\ 
            Degrees of freedom \\
            Model size reduction} \\
\midrule
Compress-train                        & 
\pbox{20cm}{Coarse-Graining then\\ 
            finding ground state}     &
\pbox{20cm}{Model size reduction \\
            then optimal weights}  \\
\midrule
Gibbs Entropy & Thermodynamic Entropy & Logarithm of accessible phase space \\
\midrule
Randomness & Thermal Fluctuations & Random process \\
\midrule
Complexity & Degrees of Freedom & Number of independent weights \\
\midrule
Performance & Thermodynamic Efficiency & Useful work, accuracy \\ 
\midrule
Clip-Pruning & Remove interaction & Relearning after removed weights \\  
\midrule
Compressed Sensing (CS) & Reduced measurements (RM) & Sparse recovery (SR) \\  
\bottomrule
Dual tomographic compression & Two-sided RM & Two-sided SR \\  
\bottomrule
\end{tabular}
\caption{Conceptual mapping between statistical physics and machine 
learning with statistical signal processing. These mappings are 
not too strict, only to give a basic idea as a summary in our context.}
\label{tab:concept}
\end{table}

Reducing size of neural networks has attracted research interest since their 
inception, such as earlier works on Hopfield networks \cite{janowsky89} in 
statistical physics community. This pioneering work identifies {\it clipping} 
as removing existing connections and {\it pruning} re-learning, i.e., fine-tuning 
after clipping. Similar clipping approach first applied in computer science by
LeCun et al. \cite{lecun89}, fun title of brain damage is used, reflecting the 
performance degradation due to clipping, hinting at magnitude pruning, so called 
{\it small saliency} that small weights don't affect the performance that much, also
recently for semi-supervised learning \cite{shwartz2024compress}. With the advent of 
deep learning \cite{lecun15a, schmidhuber15a} and generative models producing ever 
increasing size large deep learning architectures due to attention \cite{vaswani17} 
and diffusion \cite{sohl15} mechanisms, necessitates producing much smaller deep 
learning models for both training and inference efficiency, as energy consumption 
and ethical use is raised as the most important aspect in industry and academic 
practice \cite{menghani23, zhu25} and sparse model size reduction post-model 
training became important \cite{hoefler2021sparsity}, for dynamic cases as 
well \cite{han2021dynamic, tatarnikova2025optimization}. {\it Lottery ticket 
hypothesis} \cite{frankle18}, provides a deeper perspective in compressing or 
finding smaller deep learning architectures, via magnitude pruning, i.e., brain 
damage, stating that subnetworks are possible that give similar performance as 
the initial deep architecture due to initialization chance, i.e., a lottery ticket.
Applying the lottery ticket approach, iteratively pruning during training,
leads to train-compress procedure, a compression process, which is the main example 
tool for our quest for finding relationship between randomness and compression. Such 
process of train-compress (or compression aware training) as a meta algorithm has 
been already proposed \cite{han15, carreira18, zimmer25}, before lottery tickets were 
proposed. Effectiveness of random only pruning also reported recently \cite{gadhikar23}.
Using compressed sensing in pruning recurrent networks is also explored, however it 
was only on to reduce size of independently appearing network components not on 
connections \cite{zhen21a}. Using deep learning within compressed sensing framework 
is also explored recently in a different angle \cite{wu2019deep, machidon2023deep}.
The contributions of the paper are presented in the following sections {\it Inverse 
Compressed Sensing}, {\it Dual Tomographic Compression} and the main proposition, {\it 
Gibbs randomness-compression proposition}. These are supported by in-depth learning 
task on canonical vision deep learning test system as a baseline experimental 
evidence.

\begin{figure}[h!]
\centering
  \includegraphics[width=0.48\textwidth]{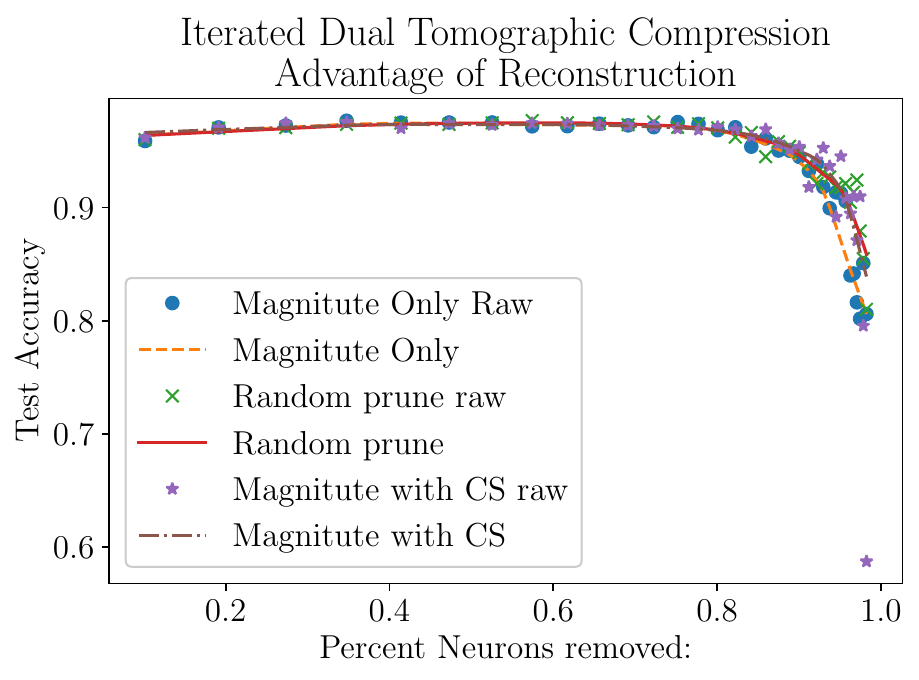}
  \caption{DTC accuracy over different sparsity levels in train-compress 
  iterative pruning. }
  \label{dtc:accuracy}
\end{figure}

\begin{figure}[h!]
\centering
  \includegraphics[width=0.48\textwidth]{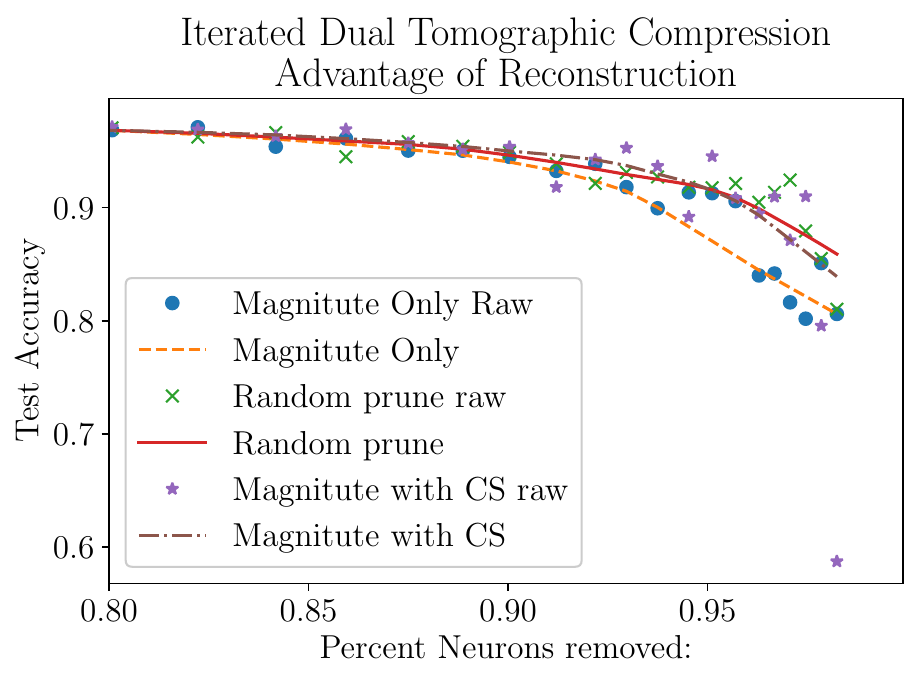}
  \caption{DTC accuracy over different sparsity levels in train-compress 
  iterative pruning. Inset shows a magnified view of higher sparsity levels.}
  \label{dtc:accuracy}
\end{figure}

\section{Inverse Compressed Sensing} \label{sec:inversecs}

We apply compressed sensing on the learned weights. However, as weights 
are known, instead of measuring the signal (weights), we do have it. This
implies so called {\it inverse compressed sensing}, not to be confused with 
inverse problem nature of the original compressed sensing framework. Inverse 
here implies, we construct an hypothetical measurement process from known 
signal, and try to reconstruct the signals sparse version via conventional 
compressed sensing, i.e.,so called  {\it weight rays}. In case of deep learning 
layers, total weights associated to neuron's input and output will be used 
instead of single CS, dual nature of tomographic process. Essentially, we 
found a sparse projection of the learned weights, which might be sparser due to
compressed sensing framework. In the notation, we will use weights, instead of 
calling a signal vector \cite{baraniuk07}.  

Given, learned weights vector $w \in \mathbb{R}^{m}$, that is flatten from any
layer matrix without any dependency on the layer type, representing all neurons 
state at a given learning time. Our purpose is to find sparse projection of 
$w$, $w_{s} \in \mathbb{R}^{n}$ where by $n$ is defined via sparsity level 
$s$, $n = m \cdot s$. Within CS framework, we define two matrices, one measurement
matrix $\Psi \in \mathbb{R}^{mxm}$, here discrete cosine transform matrix and 
randomization matrix, here normal Gaussian, $\Phi \in \mathbb{R}^{nxm}$. Then CS 
matrix reads $\Theta=\Phi \cdot \Psi \in \mathbb{R}^{nxm}$. The inverse bit as we 
discusses, comes to play when we generate a hypothetical measurement vector, 
$y = \Phi \cdot w$ as a measurement vector.  Then we minimize the following 
unconstrained version of CS procedure, to reconstruct the learned weights, 
so called {\it weight ray} $w_{r}$,

$$ \min_{w_{r}}  \left\lVert \Theta w_{r} - y \right\rVert^{2}  + 
   \lambda \left\lVert w_{r} \right\rVert_{1} $$

Obtaining different $w_{r}$ given different sparsity levels $s$ corresponds 
to neuronal level pruning in our canonical example, if magnitudes ranked in 
some manner. However, such inverse compressed sensing procedure can be applied
to any compression process whereby hypothetical measurement could be generated.

\begin{figure}[h!]
\centering
  \includegraphics[width=0.48\textwidth]{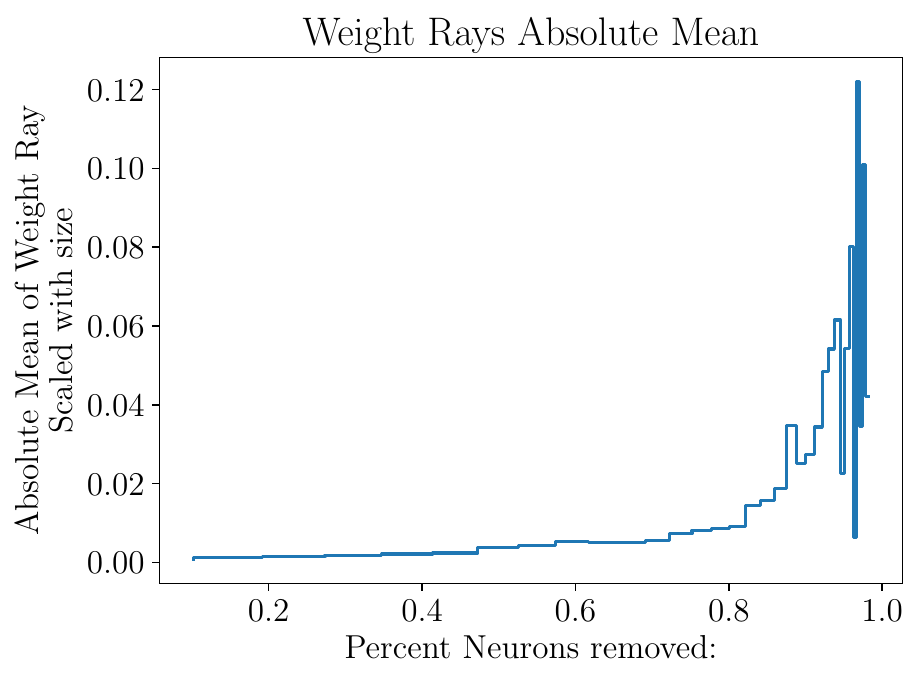}
  \caption{Weight-rays, total of reconstructed weight over different sparsity 
  levels.}
  \label{rays}
\end{figure}

\section{Dual Tomographic Compression: Neuronal-level pruning} \label{sec:dtc}

Compressed sensing approach proposed in an inverse fashion can be utilized 
for neuronal-level pruning. This can be achieved by applying inverse compressed 
sensing in a dual fashion on a layer with $n$ neurons. Weight vector can be 
obtained as follows:  a fully-connected previous layer with $p$ neurons, weight 
matrix $W_{p} \in \mathbb{R}^{pxn}$ and the next layer with $q$ neurons, weight 
matrix $W_{q} \in \mathbb{R}^{nxq}$, are used to sum over rows, resulting per 
neuron values in obtaining {\it weight rays} $w_{p} \in \mathbb{R}^{n}$ and $w_{q} \in \mathbb{R}^{n}$. 

These weight vectors are reconstructed via inverse compressed sensing, in dual fashion:
$$ \min_{w_{r}^{p}}  \left\lVert \Theta w_{r}^{p} - y_{p} \right\rVert^{2}  
   + \lambda \left\lVert w_{r}^{p} \right\rVert_{1} $$
$$ \min_{w_{r}^{q}}  \left\lVert \Theta w_{r}^{q} - y_{q} \right\rVert^{2}  
   + \lambda \left\lVert w_{r}^{q} \right\rVert_{1} $$

Then the reconstructed sparse projections summed, $w_{r}^{p}+w_{r}^{q}$ are 
used to clip neurons where their weight rays' sum is lower than the target 
sparsity. This is achieved via clipping at the  quantile value corresponding 
to the  sparsity level, between $0.0$ and $1.0$, from no sparsity to higher 
sparsity values, respectively. This compression is applied at each learning 
cycle, $m$ batches, forming the so-called train-compress cycle. After clipping,
we retain the weights from previous learning cycle and continue to train further 
at this new reduced number of neurons as a pruning step, i.e., fine-tuning.
This entire iterative pruning procedure is called {\it Dual Tomographic 
Compression (DTC)}.

\subsection{Interactions between consecutive reconstructions}

Here, we take one of the simplest interactions as $w_{r}^{p}+w_{r}^{q}$ for 
clipping neurons in a ranked manner. However, as this is a framework more 
complex interactions could be developed further but our purpose here is to 
establish the DTC's utility as a framework. 

\subsection{Layer types and propagation over layers}

We express the framework for {\it fully-connected} layers but DTC doesn't 
depends on the layer types within the neural network. Only requirement is 
mapping layer into {\it weight rays}, i.e., mapping incoming and outgoing 
connections into vectors. This can be achieved trivially by flattening any 
tensor representations of the layers to a one dimensional vector for DTC 
procedure, such as convolution or transformer architectures. 

DTC should be applied to each hidden-layer back-to back, so called in 
forward propagating the DTC. Not to be confused with backpropagation. 
This is a simple propagation of the DTC procedure. In the case of multiple 
hidden-layers, DTC procedure becomes easier as incoming connections would 
already pruned due to clipped neurons from the previous layer. 

We present this train-compress cycles in Algorithm 1 as an overall summary. 
Algorithm description aimed at presenting a representation of the framework 
with its logical components, rather than implementation details, in order to 
convey the core idea.

\begin{algorithm}[!ht]
\DontPrintSemicolon
  \KwInput{$N$ number of learn-compress cycles, $L$ hidden-layers with 
   $j$ as index, $n_{j}$ number of neurons at layer $j$, {\it weight rays} 
   $w_{j}^{p}$ and $w_{j}^{q}$, $s_{i}$ target rate of pruning
   (decreasing monotonically)
  }
  \KwOutput{$L$ hidden layers with reduced size}
  \KwData{Training set $X_{tr}$, testing set $X_{te}$}
  \tcc{Learn-compress cycles: Within batch learning, 
    smaller network at each batches}    
  \For{$i=1,...,N$} 
  {
    Train the network on the given batch(s) taken from $X_{tr}$\\
    \tcc{Layer-wise compression: DTC propagating over layers.}    
    \For{$j=1,...,L$}
    {
      Inverse Compressed sensing: \\
      Compute measurement matrices $\Theta$ and measurements vector $y$. \\
      Dual $\ell_{1}$ minimization, find weights rays $w_{j}^{p}$ and $w_{j}^{q}$. \\
      Clipping the layer (neurons) based on quantile at $s_{i}$ percent.  \\
      Ranking is achieved with using $w_{j}^{p}+w_{j}^{q}$ in order. \\
      Compress the layer: Update layers with kept neurons, updating $n_{j}$ \\
    }
    \tcc{For L >1 take mean of all Gibbs entropies}    
    Compute Gibbs entropy $G_{i}(s_{i})$ over measurement vectors $y_{j}$. \\
    Compute the performance $f_{i}(s_i)$ of the compressed network on the 
    test set $X_{te}$
  }
  \caption{An algorithmic summary of Dual Tomographic Compression (DTC)
   based train-compress cycles. This is a representation of the general
   approach in the framework. }
\end{algorithm}

\subsection{Computational Complexity}

An assessment of the computational complexity of applying DTC follows. The primary 
compute overhead due to DTC are  optimization cycles. This burden is compensated by 
the train-compress approach. Because in this framework we don't need to complete entire 
training before applying the compression, recall {\it compression aware training}. 
We quantify this by the number of optimization calls over cycles against if we were 
to run full training and compression. $M$ cycles of compression and $N$ batches 
for learning and compressions, number of Layers $L$.  

\begin{itemize}
\item DTC with train-compress \\ 
      Order of $N \cdot L$ calls to optimizers both learning and compression 
      in batch learning. Note that in this case $N=M$ as we embed the compression 
      within the training loop.
\item Full train compress without DTC \\
      Order of $N \cdot M$ calls to optimizers for learning and assuming 
      compression wouldn't need an optimizer call.  
\end{itemize}

We have deduced that, even though DTC brings some optimization load, 
computational complexity is compensated in a manner that, DTC framework is 
less compute heavy compare to full training and compress frameworks 
comparatively as $M$ compression cycles are usually much larger than the 
number of layers $L$ due to gradual compression, i.e., because compressing 
in large reductions in small steps are quite unlikely.

\section{Gibbs randomness-compression proposition} \label{sec:gibbs}

The proposition is inspired by DTC's formulation in the previous section. 
Given a compression process of $n$ steps on a vector $w$, it progresses 
with measurement vectors $y_{i}$, $i=1,...,n$ in an iterative fashion
with a random mechanism ${\bf R}$, as each step corresponds to a compression 
ratio $s_{i}$ and a fitness value, i.e., performance metric $f_{i}$. The sizes 
of measurement vectors are monotonically decreasing as the compression ratio 
compares to the initial signal. A normalized distribution of measurements is 
formed and its Gibbs entropy is computed as $G_{i}$, i.e., from normalized 
histogram bins of $k$, where $G_{i} = - \Sigma_{k} p_{k} log_{2} p_{k}$.

The Gibbs randomness-compression proposition states that there is an almost 
perfect relationship between the two functions $(s_{i}, f_{i})$ and 
$(s_{i}, G_{i})$, denoted by $f_{i}(s_{i})$ and $G_{i}(s_{i})$. This 
proposition establishes that {\it directed randomness}, because of its 
iterative nature,  within the compression process is related to compression 
performance. As compression with preserved fitness is related to the level 
of complexity, the proposition establishes a relationship between randomness 
and complexity. In other words, a lossy compression process is equivalent 
to directed randomness that preserves information content, as evidenced by 
the retention of performance or fitness. 

\subsection{Analysis of the proposition}

More formally, we express the idea in terms of relationship between 
two {\it comonotonic} functions, being highly correlated.

\begin{theorem}[Gibbs-randomness proposition]
Given {\it sequential lossy compression process} consists of $M$ 
consecutive compression cycle. Each compression cycle reduces the data 
size $s_i$ percent, representing model size,  with relative to uncompressed 
size, $i=1,..,M$, where $s_{0}=1.0$. At each compression cycle, a measure 
vector ${\bf y}_{i}$ due to {\it a randomized algorithm} with associated Gibbs
entropy $G_{i}(s_{i})$ and a corresponding learning performance $f_{i}(s_i)$ 
after the compression cycle i.e., performance on the reduced model, are generated.
If learning  performance degradation and entropy reduction due to compression kept 
small in consecutive cycles $|f_{i}-f_{i-1}| < \epsilon$ and 
$|G_{i}-G_{i-1}| < \delta$ during entire learn-compress cycles, 
then functions $f_{i}$ and $G_{i}$ are highly correlated.  
\end{theorem}

\begin{proof}
A logical proof follows, functions $f_{i}(s_{i})$ and $G_{i}(s_{i})$ are 
non-increasing monotonically (comonotonic) due to quasi-static 
decreases $\epsilon_{i}$ and $\delta_{i}$ locally, while mean 
differences $\Delta$ are close to zero, 
$\Delta(\epsilon_{i})=\Delta(\delta_{i}) \rightarrow 0$ 
as $|s_{i}-s_{i-1}| \rightarrow 0$. These functions are {\it comoving} 
over compress-learn cycle (functions evolve synchronously), hence are highly 
correlated.
\end{proof}

We known that functions $f_{i}(s_{i})$ and $G_{i}(s_{i})$ are
{\it mechanistically connected} due to fact that measurement vector ${\bf y}_{i}$ 
is defined on the network weights, and learned weights directly 
effect the network learning performance. Because of random process, here Dual 
Tomographic Compression (DTC) generates ${\bf y}_{i}$, hence the theorem establishes 
the connection between randomness and compression.

The theorem is expressed for a single hidden-layer but it is readily applicable 
to entire neural network by simplify repeating it to all hidden-layers back to back 
at a given compress-train cycle. It can be applied to any layer type by mapping 
the weight structure to measurement vector ${\bf y}_{i}$. This is demonstrated 
in the Algorithm 1 in detail.

\begin{figure}[h!]
\centering
  \includegraphics[width=0.48\textwidth]{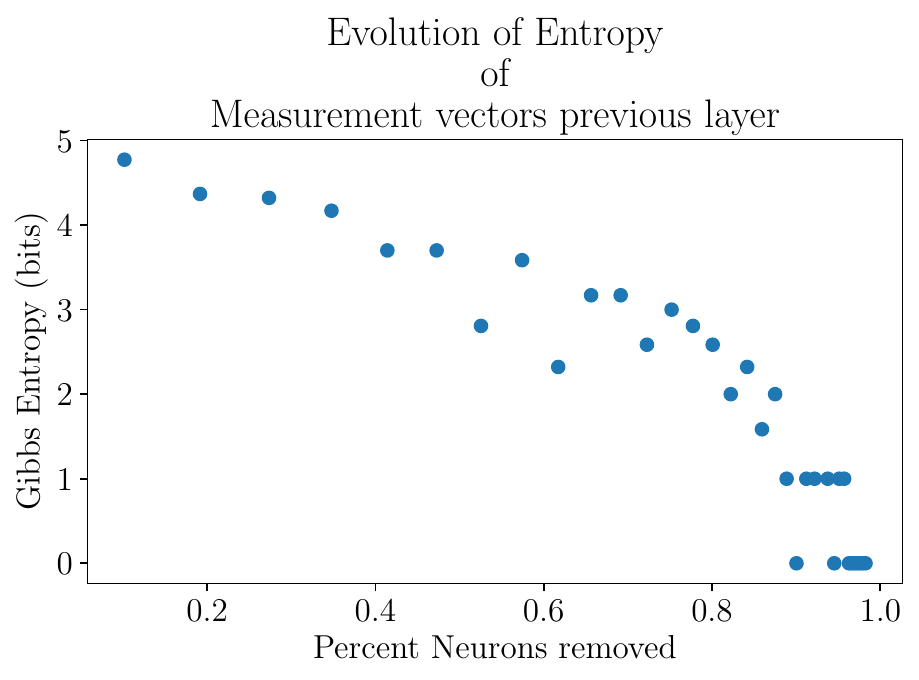}
  \caption{Previous layer Gibbs entropy of computed measurements within DTC over compression process.}
  \label{gibbs1}
\end{figure}

\section{Experiments and discussions} \label{sec:exp}

We have embedded DTC into a canonical vision task, classifying MNIST 
dataset \cite{lecun02} via softmax output. A deep learning model is 
implemented with hidden layer size of 512 neurons as starting point for 
train-compress iterative pruning. 

We apply dual compressed sensing at every 200 batches of size 512, 
triggering inverse compressed sensing procedure using SCS algorithm 
\cite{odono16} for minimizing for {\it weight rays}, for simplicity 
we didn't use bias terms. The selection of 200 batches were made such 
that performance is increasing sufficiently, even it is gradual, and 
limited by the trade-off between compute budget and compressing in a 
meaningful rate.

Hyperparameter and related details include Adam optimizer's learning 
rate of $0.001$, $ReLU$ activations, fixed {\it Xavier} initialization, 
batch normalization for the baseline network architecture with fully 
connected layers. We utilize PyTorch \cite{paszke17} for deep learning and 
cvxpy \cite{diamond16} for DTC steps. 

In the conducted experiments,  two more different clipping approach 
is implement for fair comparison. Namely random and magnitude pruning. 
Random clipping is just randomly drops neurons at the specified 
sparsity levels. In the  case of magnitude pruning we simply rank 
the highest weight inputs to given neurons. In total we run three 
compress-train cycles to compare the test performances, as summarized
in Table 1. 

\begin{table}[h]
\centering
\caption{Test accuracy for different pruning approaches over different 
sparsity levels in iterative pruning.}
\label{tab:pruning_comparison}
\begin{tabular}{llll}
Sparsity & Magnitude & Random  & DTC  \\
0.1016 & 0.9638 & 0.9637 & 0.9663 \\
0.1914 & 0.9675 & 0.9669 & 0.9687 \\
0.3477 & 0.9735 & 0.9723 & 0.9728 \\
0.4141 & 0.9745 & 0.9736 & 0.9735 \\
0.5254 & 0.9740 & 0.9747 & 0.9732 \\
0.6172 & 0.9730 & 0.9748 & 0.9732 \\
0.7520 & 0.9718 & 0.9719 & 0.9705 \\
0.8594 & 0.9560 & 0.9588 & 0.9609 \\
0.9121 & 0.9324 & 0.9399 & 0.9463 \\
0.9219 & 0.9234 & 0.9341 & 0.9423 \\
0.9297 & 0.9144 & 0.9292 & 0.9371 \\
0.9512 & 0.8703 & 0.9157 & 0.9153 \\
0.9707 & 0.8293 & 0.8834 & 0.8717 \\
0.9785 & 0.8138 & 0.8673 & 0.8500 \\
0.9824 & 0.8060 & 0.8589 & 0.8391 \\
\end{tabular}
\end{table}

\subsection{Learning performance over compression cycles}

We record the test accuracy of the reduced architecture after 
each train-compress cycle for different pruning approaches: 
magnitude only, random and DTC based magnitude pruning. As shown 
in Table 1, Figure 1 and Figure 2. DTC and random pruning perform 
superior. We observe that DTC shows a good performance even against 
random pruning after 85 percent sparsity. However, random pruning 
is shown remarkable resilience. But we have clearly shown that DTC 
outperforms magnitude pruning.  

For very high sparsity levels after 95 percent, randomness shows 
a bit better performance. This isn't related to network initialization, 
a fixed Xavier was used making results on the table fully repeatable. 
While compression process is robust to initialization that performance 
fluctuations are negligible, rather this is due to fact that after 95 percent 
compression network remain with handful of neurons in the hidden layer. 
Depending on requirements very good performance is also retained even
after 95 percent of the neurons are pruned.

\begin{figure}[h!]
\centering
  \includegraphics[width=0.48\textwidth]{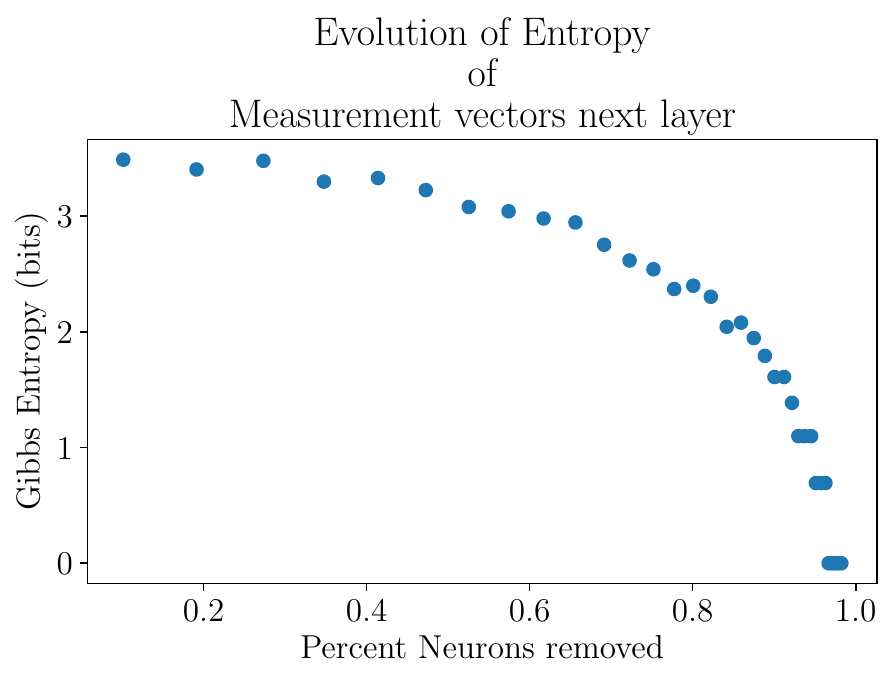}
  \caption{Next layer Gibbs entropy of computed measurements within DTC 
  over compression process.}
  \label{gibbs2}
\end{figure}

Total weight rays in Figure 3 are shown to be more pronounced for very 
high sparsity levels. This signifies, attenuation of learned weight as 
model gets much smaller to compensate the performance loss. This is 
{\it a striking result} as DTC sparsify the weights via inverse 
compressed sensing.

\subsection{Evidence for the proposition}

We also reported Gibbs entropy over train-compress iterative pruning
in Figures 4 and 5. As complexity reduces, Gibbs entropy over measurement 
vectors also decreases. Pearson correlations between these and performance 
reduction are computed for DTC and random pruning. They read
0.9174 and 0.9412 respectively, empirically supporting our proposition.
This result demonstrated a strong empirical support for the Theorem 1, 
Gibbs-randomness compression proposition. While both DTC and random pruning 
have obviously random mechanism the Gibbs-compression proposition requires. 
As random pruning correlations is much higher, this strengthen the proposition's 
statement. Moreover comonotonic behavior is also verified by the visual 
inspection between Figure 1 and 3 versus Figure 4 and 5.

\section{Conclusion} \label{sec:con}

A quantitative relationship between randomness and compression
is proposed. The proposition relies on comonotonic relationship
between performance of compression and the degree of randomness 
measured on the a learning model being compressed. We have 
demonstrated that there is indeed almost perfect comonotonic 
relationship between randomness and compression for a learning
system as a showcase on a baseline model. This finding contributes to the 
current discussions on how compression and learning 
models are related and the role of randomness in learning 
from statistical physics perspective.  

\section*{Acknowledgment}
We are grateful to Y.S\"uzen for her kind support of this work. 
Author is also grateful for constructive comments from 
the {\it Neural Networks} community, strengthen the paper's 
technical case and presentation.

\bibliographystyle{iopart-num}
\bibliography{suzen}

\end{document}